\documentclass[11pt]{article}

\usepackage[T1]{fontenc}
\usepackage[utf8]{inputenc}
\usepackage[margin=1in]{geometry}

\usepackage{amsmath}
\usepackage{amssymb}
\usepackage{amsthm}
\usepackage{mathtools}

\usepackage{natbib}
\usepackage{hyperref}
\hypersetup{
    colorlinks=true,
    linkcolor=blue,
    citecolor=blue,
    urlcolor=blue
}

\newcommand{\Eta}{\mathcal{H}}
\newtheorem{theorem}{Theorem}[section]
\newtheorem{corollary}{Corollary}[theorem]
\newtheorem{lemma}{Lemma}[section]

\title{Hierarchical Solomonoff Induction: An Unbounded Machine Learning Model}
\author{
    Nathan Young\\
    Strong AI Lab, University of Auckland\\
    \texttt{nathan.young@auckland.ac.nz}
}
\date{}

\begin{document}
\maketitle

\begin{abstract}
Solomonoff Induction, or SolInd, provides an ideal unbounded model of a priori sequence prediction but cannot naturally describe extrapolation from a given training dataset, as performed by Large Language Models.
We apply de Finetti's theorem on exchangeable distributions to SolInd to produce what we call Hierarchical Solomonoff Induction, or HSI, which maintains a hyperprior over all Solomonoff priors that can be conditioned on previously observed sequences.

We extend Wood et al.'s proof that universal mixtures of semimeasures are equivalent to SolInd to show that universal mixtures of these mixtures are also equivalent, proving that HSI=SolInd.
We also prove that HSI's excess error on any distribution, compared to its true generator, is bounded by that generator's complexity in the hyperprior. This result is directly comparable to SolInd's prediction error being bounded by the Kolmogorov complexity of the sequence being predicted, and forces HSI's average excess error to converge to 0 as a dataset grows, leading to optimal prediction in the limit.
We claim that HSI is an ideal unbounded model of sequence prediction given a dataset in the same way that SolInd is ideal over individual sequences.
\end{abstract}

\section{Introduction}
To understand a solution to a computational problem, it can be helpful to first understand that problem's unbounded solution, i.e. how it might be solved with unlimited computing power.
Modern Large Language Models (LLMs) operate in the domain of \textit{sequence prediction}: given an ordered sequence of symbols on some alphabet (e.g. letters, Unicode characters, or tokens), what symbol comes next? Or, in a probabilistic setting, what distribution over symbols should be expected to come next?

The optimal solution to both of these problems is \textit{Solomonoff Induction}, or \textit{SolInd}, first described by~\citet{Solomonoff1964}.
SolInd is Bayesian induction over inputs to a universal Turing machine $U$, with each input $s$ assigned a prior probability of $2^{-\ell(s)}$ (i.e. shorter inputs get more weight).

SolInd has been shown~\citep{Hutter2001a} to have several useful optimality properties, including universal dominance (no computable predictor can beat SolInd on any sequence by more than a constant factor) and bounded error (SolInd's prediction error on any sequence is bounded by that sequence's Kolmogorov complexity on $U$).

There has been some recent research on modelling LLMs as approximations of Solomonoff Induction~\citep{Young2025, Wan2024}.
However, this research misses a significant part of how LLMs, and machine learning (ML) systems in general, work.
Solomonoff inductors make optimal predictions over individual sequences, but not over \textit{datasets}; it is possible to concatenate a set of strings and use this as a prefix for the sequence that a Solomonoff inductor must predict, but this has the risk of introducing new information not present in the sequences (for example, if the strings are sorted) and identifying programs more complicated than the dataset's generator.
In short, SolInd cannot be \textit{trained}, and serves as an ideal model of sequence prediction at test time only; it cannot describe ideal performance at training time.
Even then, the distribution over sequences represented by the UTM $U$ does not necessarily correspond to the distribution we wish to predict, and it cannot be updated from one sequence to another.

With the aim of introducing an unbounded sequence prediction model that can operate over arbitrary datasets of sequences as well as individual, we propose an alternative formulation called \textbf{Hierarchical Solomonoff Induction}, or \textbf{HSI}.
HSI is a generalisation of SolInd to computable distributions and resembles a hierarchical Bayesian model, maintaining not only a prior over sequences but a \textit{hyperprior} $\Eta$ over all computable distributions of sequences (equivalently, over all Turing machines or all Solomonoff priors).
HSI is identical to SolInd on individual sequences, but can also condition its hyperprior on a given dataset, in a way analogous to ML training.
HSI exhibits optimality properties over datasets that are comparable to SolInd's over sequences, making a finite amount of error on any dataset bounded by that dataset's generator's complexity in $\Eta$ and dominating any other predictor of datasets.

\section{Setup}

\subsection{Preliminaries and Notation}
Unless otherwise stated, we follow the definitions and notation of \citet{Wood2013}.
We define the binary alphabet $\mathbb{B}=\{0,1\}$, the set of binary strings $\mathbb{B}^*$ and infinite binary sequences $\mathbb{B}^\infty$. We denote the empty string as $\epsilon$, the length of a string $x$ as $\ell(x)$, the concatenation of strings $a$ and $b$ as $ab$, and prefixes as $\sqsubseteq$ (that is, $p\sqsubseteq px$ for any $x$).
We then have the cylinder set of $x$, $\Gamma_x=\{\omega\in\mathbb{B}^\infty:x\sqsubseteq\omega\}$, and the uniform measure $\lambda$ with $\lambda(\Gamma_x)=2^{-\ell(x)}$.

Rather than full probability measures over sequences with $\mu(\epsilon)=1$ and $\mu(x)=\mu(x0)+\mu(x1)$, we consider \textit{semimeasures}, which permit $\mu(\epsilon)\leq 1$ and $\mu(x)\geq\mu(x0)+\mu(x1)$.
A semimeasure is \textit{lower-semicomputable} (equivalently, \textit{enumerable}) if it can be approximated from below by a computable, monotonically increasing series of values. We write $\mathcal{M}$ for the set of all enumerable semimeasures.

All Turing machines (TMs) discussed in this paper are of a class called \textit{monotone Turing machines}. A monotone TM $M$ has a unidirectional read-only input tape on which is written some sequence $s$, a unidirectional write-only output tape on which it writes some sequence $x$, a bidirectional working tape, and a finite state machine defining its behaviour over these tapes.
$M$ cannot proceed if it attempts to read from a blank cell on its input tape; it must wait for input to arrive. Therefore, if $M(s)$ is finite, $M$ may produce more output if given more input.
If $M$ outputs $x$ when given $s$ as input, we write that $M(s)=x$. For any $x'\sqsubseteq x$ we say that $s$ \textit{codes for} $x'$ on $M$ and write $M(s)\sqsupseteq x'$.

We then define, over a TM $M$, the Solomonoff semimeasure $P_M$:
\[
P_M(x) \coloneq \lambda(\{\omega:M(\omega)\sqsupseteq x\})
\]
This is equivalent to the measure of all minimal $s$ that code for $x$ on $M$:
\[
P_M(x) = \sum_{\lfloor s:M(s)\sqsupseteq x\rfloor} 2^{-\ell(s)}
\]
where $\lfloor X \rfloor$ is the maximal prefix-free subset of $X$.
By Lemma 6 of \citet{Wood2013}, Solomonoff priors on monotone TMs correspond directly to the class of lower-semicomputable semimeasures.\footnote{This only holds for nonempty strings, as TMs begin with an empty output tape and so always assign $P_M(\epsilon)=1$. In this paper, we ignore the case of the empty string when comparing semimeasures.}

\subsection{Universal TMs}
For a Solomonoff semimeasure to be a universal inductor, the chosen TM must be a universal Turing machine, or UTM.
Existing literature~\citep{Wood2013, LiVitanyi2019} describes multiple classes of TMs that might be considered `universal'; here we disambiguate and examine three such classes.

\begin{itemize}
    \item \textbf{Universality by surjection:} A TM $U$ is \textit{surjective} if, for any output sequence $x$, there exists some input sequence $s$ such that $U(s)=x$ (equivalently, $P_U(x)>0$ for all $x\in\mathbb{B}^*$).
    \item \textbf{Optimality:} A TM $U$ is \textit{optimal} if, for any TM $V$, there exists some positive constant $c$ such that $P_U(x)\geq cP_V(x)$ --- i.e. $P_U$ dominates all Solomonoff semimeasures. This is the standard notion of an (additively) optimal machine in algorithmic information theory~\citep{LiVitanyi2019, DowneyHirschfeldt2010}.
    \item \textbf{Universality by adjunction:} A TM $U$ is \textit{universal by adjunction} if there exists an enumeration $\{V_i:i\in\mathbb{N}\}$ of all TMs and a prefix-free encoding $p:\mathbb{N}\to\mathbb{B}^*$ such that $U(p(i)s)=V_i(s)$~\citep{FigueiraStephanWu2006, DowneyHirschfeldt2010} --- in other words, $U$ can emulate any TM.\footnote{We might also define a class of machines with an encoding that is not necessarily prefix-free; this class, if distinct from both optimality and universality by adjunction, would be a subset of the former and a superset of the latter.}
    In line with \citet{Wood2013}, we use the term `universal Turing machine', or `UTM', to refer to this class.

    Note that for every $V_j$ that is universal by adjunction, there is some additional prefix $p'$ such that $U(p(j)p'(i)s)=V_j(p'(i)s)=V_i(s)$; there are therefore infinitely many $p(j)p'(i)$ such that $U(ps)=V_i(s)$.
\end{itemize}
\citeauthor{Wood2013} require UTMs to give no output if their input does not begin with some $p(i)$; however, Corollary \ref{cor:utm-codes} of Theorem \ref{thm:enums-of-ums=ums} shows that this is unnecessary.

These classes form a strict hierarchy: UTMs are a subset of optimal TMs, which are a subset of surjective TMs.
We now demonstrate this fact, as it is relevant to our later discussion of the domain of HSI's hyperprior, as well as HSI's equality to SolInd and induction over TMs in general.

\begin{lemma}
\label{lm:optimal-neq-utm}
There exists an optimal TM that is not a UTM.
\end{lemma}
\begin{proof}
By Theorem 16 in \citet{Wood2013}, there exists a universally dominant enumerable semimeasure $\delta'$ that is not a universal mixture.
By their Lemma 6, there exists a TM $U_{\delta'}$ with a Solomonoff semimeasure equal to $\delta'$.
By their Lemma 10, all UTMs have Solomonoff semimeasures equivalent to a universal mixture.
$U_{\delta'}$ has a universally dominant Solomonoff semimeasure that is not a universal mixture, and is therefore optimal but not a UTM.
\end{proof}
As an intuition, $\delta'$ assigns $\delta'(0)=\delta'(1)=0.5$, equivalent to the Solomonoff semimeasure over a TM that begins by either outputting a 0 or a 1; such a machine cannot emulate a `silent' TM that never prints any output at all.

\begin{lemma}
\label{lm:surjective-neq-optimal}
    There exists a surjective TM that is not optimal.
\end{lemma}
\begin{proof}
    Consider the identity TM $U_I$, where $U_I(s)=s$. $U_I$ is surjective, as it may output any $x$ on receiving $x$ as input.
    We then define the `all-zeroes' TM % OOPS! All zeroes
    $U_0$, which continuously prints 0 on its output tape regardless of input.
    On any sequence $0^n$ of $n$ 0s, $P_{U_0}(0^n)=1$ and $P_{U_I}(0^n)=2^{-n}$.

    For $U_I$ to be optimal, the following must hold for some positive constant $c$ and all $n\in\mathbb{N}$:
    \begin{align*}
    P_{U_I}(0^n) &\geq cP_{U_0}(0^n)\\
    2^{-n} &\geq c
    \end{align*}
    Since $2^{-n}\to 0$, no $c$ can satisfy $2^{-n} \geq c$ for all $n$. Therefore $U_I$ is surjective but not optimal.
\end{proof}

\begin{lemma}
\label{lm:universal-subsets}
    Let $A$ be the set of all universal-by-adjunction TMs, $O$ the set of all optimal TMs, and $S$ the set of all surjective TMs.
    Then $A\subsetneq O\subsetneq S$.
\end{lemma}
\begin{proof}
    Take some $U_a\in A$. By definition, for any TM $T$, there exists some prefix $p$ such that $U_a(ps)=T(s)$.
    Considering only inputs to $U_a$ of the form $ps$, we have:
    \[
        P_{U_a}(x) > 2^{-\ell(p)}P_T(x)
    \]
    Therefore there exists a positive constant $c=2^{-\ell(p)}$ such that $P_{U_a}(x)\geq cP_T(x)$, so $U_a \in O$ and $A\subseteq O$.

    Further, fix any $U_o\in O$ and $U_s \in S$. Then for some positive $c$, $P_{U_o}(x)\geq cP_{U_s}(x)>0$ for any $x$. Therefore $U_o \in S$ and $O\subseteq S$.

    By Lemmas \ref{lm:optimal-neq-utm} and \ref{lm:surjective-neq-optimal}, $A\neq O$ and $O\neq S$.
    Therefore $A\subsetneq O\subsetneq S$.
\end{proof}

UTMs are required for some of SolInd's optimality properties, and are the strictest of these three criteria; we therefore define SolInd with reference to UTMs only.

\subsection{Solomonoff Induction}
We can now define the \textit{Solomonoff prior} $P_U$ with respect to a UTM $U$ as the Solomonoff semimeasure over $U$.
The use of this prior to predict a sequence is \textit{Solomonoff Induction}, or \textit{SolInd}.

While SolInd can be beaten by individual TMs on individual sequences --- for example, any SolInd will incur loss on the ``all-zeroes'' sequences in Lemma \ref{lm:surjective-neq-optimal}, while $U_0$ will not --- its logarithmic loss compared to any other TM $M$ is bounded by the complexity of $M$ on $U$.
By Lemma \ref{lm:universal-subsets}, $P_U$ dominates $P_M$, giving:
\begin{align*}
P_U(x) &> 2^{-\ell(p)} P_M(x) \\
\frac{P_U(x)}{P_M(x)} &> 2^{-\ell(p)} \\
-\log_2\bigl(\frac{P_U(x)}{P_M(x)}\bigr) &< \ell(p)
\end{align*}
for some $p$ such that $U(ps)=M(s)$.
Note that the length of the shortest such $p$ is exactly the Kolmogorov complexity $K_U(M)$ of $M$ on $U$.

This gives $P_U$ a finite error bound on every computable sequence and a constant $2^{-K_U(M)}$ by which it dominates every enumerable semimeasure $P_M$.
These bounds differ according to the choice of $U$; sufficiently poor choice of UTM can lead to arbitrarily large error bounds~\citep{Leike2015}.

\citet{Wood2013} proved direct correspondences between TMs and semimeasures:
\begin{itemize}
\item Lemma 6 proves that the class of enumerable semimeasures is equivalent to Solomonoff semimeasures over the class of TMs;
\item Theorem 14 proves that the class of universal mixtures of enumerable semimeasures, $\mathcal{U}_\xi$, is equivalent to Solomonoff semimeasures over the class of UTMs $\mathcal{U}_M$ (i.e. to SolInd);
\item Theorem 16 proves that there exist universally dominant enumerable semimeasures that are not equivalent to universal mixtures. These are instead Solomonoff semimeasures over TMs that are optimal but not universal-by-adjunction. This is another way that Solomonoff prediction can differ depending on the choice of machine; these semimeasures have the same optimality properties as SolInd, but are not Solomonoff priors.
\end{itemize}

\subsection{De Finetti's Theorem}

De Finetti's theorem~\citep{deFinetti1931}
states that if an infinite sequence of variables is exchangeable --- that is, if the sequence has equal likelihood regardless of ordering --- then there is some latent variable that makes these variables conditionally independent, and that the distribution defined by this variable can be expressed as a mixture over all possible independent and identically distributed (i.i.d.) distributions --- that is, a probability distribution over possible distributions, or in other words, a hyperprior over priors.
Applied to computable sequence distributions, this would mean that any unordered dataset can be expressed as the prefix of an infinite sequence of independent outputs of a computable sequence distribution,
which can be modelled as a mixture of all such distributions.

Not only does this framework match ML in describing how to extrapolate from unordered datasets, but the process of searching through the space of distributions to find the one that matches the underlying distribution is the ultimate goal of any ML system.
A generalisation of SolInd to the task of sequence prediction by extrapolation from datasets may therefore be obtained by applying de Finetti's theorem to Solomonoff's formulation.

It has been shown~\citep{HewittSavage1955} that de Finetti's theorem holds for Polish spaces, including Cantor space (i.e. $\mathbb{B}^\infty$); therefore, the exchangeable distribution from which a dataset of strings is drawn (and therefore its ideal predictor) must take the form of a hyperprior over i.i.d. component distributions.
While the computability of the de Finetti measure has been established only for real-valued sequences~\citep{FreerRoy2009, FreerRoy2012}, we may assume for the moment that it extends to enumerable semimeasures over strings, and examine what properties such a hierarchical distribution would have.

\section{Hierarchical Solomonoff Induction}

\subsection{Definition}

Hierarchical Solomonoff Induction, or HSI, defines a distribution over Solomonoff priors, with a \textit{hyperprior} $\Eta$ assigning weight to all possible UTMs.

Formally, let $\{U_i\}_{i\in\mathbb{N}}$ be an enumeration of TMs in which every UTM is guaranteed to appear and $\Eta: \{U_i\} \to [0,1]$ be an enumerable semimeasure over this enumeration. Define:
\begin{equation}
\label{eq:hsi}
    P_\Eta(x) \coloneq \sum_i \Eta(U_i) P_{U_i}(x)
\end{equation}
That is, $P_\Eta$ is a weighted mixture of Solomonoff priors $P_{U_i}$.
We assume for the moment that every $U_i$ is a UTM so that $P_\Eta$ is a mixture of Solomonoff priors.

After some $x$ has been observed, the remaining weight in the hyperprior is equal to
\begin{equation}
\label{eq:eta-sequence}
    \Eta(U_i \mid x) = \Eta(U_i) P_{U_i}(x)
\end{equation}
This \textit{conditioned} hyperprior will be used to predict the next sequence $y$:
\begin{equation}
\label{eq:hsi-sequence}
\begin{split}
    P_\Eta(y \mid x) &= \sum_i \Eta(U_i\mid x) P_{U_i}(y)\\
    &= \sum_i \Eta(U_i) P_{U_i}(x) P_{U_i}(y)
\end{split}
\end{equation}

It is worth considering another form of HSI with the conditioned hyperprior being renormalised:
\begin{equation}
\label{eq:eta-normalised}
    \Eta(U_i \mid x) = \Eta(U_i) \frac{P_{U_i}(x)}{P_\Eta(x)} = \frac{\Eta(U_i) P_{U_i}(x)}{\sum_j \Eta(U_j) P_{U_j}(x)}
\end{equation}
This form has the advantage that its likelihoods will not monotonically decrease, but its conditional (i.e. next-symbol) predictions are identical, and the division of a lower-semicomputable measure by another removes lower-semicomputability.
We therefore consider the unnormalised form in this paper.

Here we consider a hyperprior over an enumeration of \textit{UTMs}. This is the natural extension of SolInd to a hierarchical Bayesian context, with some weight on each possible Solomonoff distribution.
However, there are two potential issues with this formulation.
Firstly, universality (in all three senses described above) is undecidable in general. We may enumerate TMs that can be proven to be universal, but never the entire class.
Secondly, our assumption of computable de Finetti distributions makes no mention of universality, only that the true generator was an enumerable semimeasure. The equivalent de Finetti mixture may therefore include non-universal semimeasures.

We may then want to define $\Eta$ over all TMs. This would mean that it included all UTMs with no decidability issues; however, this would make it a universal mixture of enumerable semimeasures, i.e. no different from SolInd.
Conditioning the hyperprior on any string $x$ would assign $\Eta(U\mid x)=0$ to many non-universal TMs, eliminating their contributions forever, but never to all, since for any finite set of strings $X$ there must be some non-universal TM that outputs $X$.

We must therefore consider what class of TMs $\Eta$ should be defined over, and how these choices affect its properties, including its relation to SolInd.
An $\Eta$ defined over all UTMs is also a mixture over the enumerable semimeasures that define those UTMs, making such an HSI an instance of SolInd; but can any SolInd be expressed in the form of this more restrictive hyperprior?
This question may be expressed in more general terms: define the class of universal mixtures \textit{of} universal mixtures
\[
    \Xi(x) \coloneq \sum_i w(i)\xi_i(x)
\]
with some enumeration $\{\xi_i\}$ of the class of all universal mixtures $\mathcal{U}_\xi$ and an enumerable weight function $w$.
Denote the class of all $\Xi$ as $\mathcal{U}_\Xi$.
Is $\mathcal{U}_\Xi$ equal to $\mathcal{U}_\xi$?
Just as $\Eta$ over universal mixtures is equal to some SolInd, every member of $\mathcal{U}_\Xi$ is in $\mathcal{U}_\xi$; if the reverse is true, then the class of HSI instances $\mathcal{U}_\Eta$ is equal to the class of Solomonoff inductors $\mathcal{U}_M$.

In fact, it is possible to prove a more general theorem --- that all enumerations of semimeasures in which all universal mixtures appear, when assigned arbitrary weights, are equivalent.

\subsection{Equivalence of HSI and SolInd}

\begin{theorem}
\label{thm:enums-of-ums=ums}
Let $\{\nu_i\}_{i=1}^\infty$ be any enumeration of enumerable semimeasures in which every universal mixture $\xi\in\mathcal{U}_\xi$ appears.
Let $\mathcal{C}_\nu$ be the class of all mixtures $c\coloneq\sum_{i\in\mathbb{N}} w^{(c)}(i)\nu_i$ whose weight functions $w^{(c)}:\mathbb{N}\to\mathbb{R}$ are lower-semicomputable with $w^{(c)}(i)>0\;\forall i\in\mathbb{N}$ and $\sum_{i\in\mathbb{N}}w^{(c)}(i)\leq1$.
Then $\mathcal{C}_\nu=\mathcal{U}_\xi$.
\end{theorem}
\begin{proof}
We take any $c\in\mathcal{C}_\nu$ and any $\xi\in\mathcal{U}_\xi$, and show that each can be used to construct an instance of their own class equivalent to the other.

We first prove that $\mathcal{C}_\nu\subseteq\mathcal{U}_\xi$.
Since $\nu$ includes all universal mixtures, $\xi$ must appear at some index $i_\xi^{(\nu)}$. Define a new enumeration of semimeasures $\nu'$ that alternates between the outputs of $\nu$ and $\xi$'s own semimeasure enumeration $\nu^{(\xi)}$, skipping $\xi$ in $\nu$. Note that as a universal mixture, $\xi$ appears in its own enumeration.
Give each semimeasure in $\nu$ its corresponding weight in $w^{(c)}$ (again skipping $\xi$), and give each semimeasure in $\nu^{(\xi)}$ its weight in $w^{(\xi)}$ multiplied by $w^{(c)}(i_\xi^{(\nu)})$.
This defines a measure with lower-semicomputable weights on every enumerable semimeasure --- that is, a universal mixture --- equal to $c$. Then $c\in\mathcal{U}_\xi$ and $\mathcal{C}_\nu\subseteq\mathcal{U}_\xi$.

We then prove that $\mathcal{C}_\nu\supseteq\mathcal{U}_\xi$.
Since $\nu^{(\xi)}$ includes all enumerable semimeasures, $c$ must appear at some index $i_c^{(\xi)}$.
Compute $w^{(\xi)}(i_c^{(\xi)})$ from below to find some positive rational $q\leq w^{(\xi)}(i_c^{(\xi)})$ and fix $\delta = \frac{q}{2}$ so that $0<\delta<w^{(\xi)}(i_c^{(\xi)})$.
Define a new universal mixture $\xi'$ with identical semimeasure enumeration to $\xi$ but a new weight function ${w^{(\xi)}}'$:
\[{w^{(\xi)}}'(i) = \begin{cases} \frac{w^{(\xi)}(i)}{1-\delta} \text{ if } i \neq i_c^{(\xi)} & \\ \frac{w^{(\xi)}(i)-\delta}{1-\delta} \text{ if } i = i_c^{(\xi)} \end{cases}\]
As a universal mixture, $\xi'$ must appear in $\nu$ at some index $i_{\xi'}^{(\nu)}$.
We may then define a new weight function ${w^{(c)}}'$ over $\nu$:
\[{w^{(c)}}'(i) = \begin{cases} 1-\delta + \delta w^{(c)}(i) \text{ if } i = i_{\xi'}^{(\nu)} & \\ \delta w^{(c)}(i) \text{ if } i \neq i_{\xi'}^{(\nu)} \end{cases}\]
Multiplying $\xi'$ by $1-\delta$ recovers the original weight function of $\xi$, except for the subtraction of $\delta$ from its weight on $c$; this is recovered by every $\nu_i$ receiving its weight in $w^{(c)}$ multiplied by $\delta$.
This defines a measure with lower-semicomputable weights over $\nu$ --- that is, a member of $\mathcal{C}_\nu$ --- equal to $\xi$. Then $\xi\in\mathcal{C}_\nu$ and $\mathcal{C}_\nu\supseteq\mathcal{U}_\xi$.

We then have $\mathcal{C}_\nu\subseteq\mathcal{U}_\xi\subseteq\mathcal{C}_\nu$, which gives $\mathcal{C}_\nu=\mathcal{U}_\xi$.
\end{proof}
Note that this is a non-constructive existence proof --- $i_\xi^{(\nu)}$ and $i_c^{(\xi)}$ are not always identifiable, but do always exist.

Several useful corollaries follow from this result.
Firstly, note that this means that any set of such classes must be equivalent:
\begin{corollary}
\label{cor:sets-of-enums=ums}
Let $N$ be any nonempty set of enumerations of enumerable semimeasures, each of which contains every universal mixture, and for $\nu\in N$ let $\mathcal{C}_\nu$ be the class from Theorem \ref{thm:enums-of-ums=ums}.
Then $\bigcup_{\nu\in N}\mathcal{C}_\nu=\bigcup_{\nu\in N}\mathcal{U}_\xi=\mathcal{U}_\xi$.
\end{corollary}
All such enumerations $\nu$ therefore form an equivalence class when assigned arbitrary weights.

This includes all hyperprior domains proposed above:
\begin{corollary}
\label{cor:hsi-domain-equivalence}
The class of HSI priors is identical given any hyperprior $\Eta$ with a domain that includes all UTMs; for example, hyperpriors that use optimal TMs, surjective TMs, or all TMs will give identical HSI priors. Additionally, since $\Eta(U\mid X)>0$ for any UTM $U$ and dataset $X$, this remains true for all conditioned hyperpriors.
\end{corollary}
We may therefore assume only that $\Eta$ assigns positive weight to all UTMs to obtain the same class $\mathcal{U}_\Eta$ regardless of what other machines are included.

Having resolved our concerns about $\mathcal{U}_\Eta$, we may now compare it to $\mathcal{U}_M$, and $\mathcal{U}_\Xi$ to $\mathcal{U}_\xi$:
\begin{corollary}
\label{cor:um=umum=solind=hsi}
The classes $\mathcal{U}_\xi$ of universal mixtures, $\mathcal{U}_\Xi$ of universal mixtures of universal mixtures, $\mathcal{U}_M$ of Solomonoff priors, and $\mathcal{U}_\Eta$ of universal mixtures of Solomonoff priors (that is, HSI priors) are all exactly equivalent.
\end{corollary}

HSI being equivalent to all SolInd makes it a strictly smaller class than the universally dominant semimeasures:
\begin{corollary}
\label{cor:optimal-solind>hsi}
The class of Solomonoff semimeasures on optimal TMs $\mathcal{U}_\delta$, being strictly larger than $\mathcal{U}_\xi$, is also strictly larger than $\mathcal{U}_\Eta$ regardless of hyperprior domain, even if all such optimal machines are included; that is, there exist universally dominant inductors that are not equivalent to any HSI, although these inductors do not use UTMs.
\end{corollary}

Finally, we may refine our definition of a UTM by relaxing a requirement from \citet{Wood2013}:
\begin{corollary}
\label{cor:utm-codes}
If a TM $U$ has a prefix-free encoding $p:\mathbb{N}\to\mathbb{B}^*$ for some enumeration $\{V_i:i\in\mathbb{N}\}$ of all TMs such that $U(p(i)s)=V_i(s)$, this is sufficient for $U$ to be universal by adjunction, as this includes all universal semimeasures; the requirement of \citet{Wood2013} that they produce no output on any other inputs is therefore unnecessary and does not change the class of priors defined.
\end{corollary}
This definition admits some optimal machines that \citeauthor{Wood2013} do not consider UTMs, but none with new Solomonoff priors; e.g. the optimal machine $U_{\delta'}$ from Lemma \ref{lm:optimal-neq-utm} is still unable to emulate a silent TM and so does not meet this definition.

\subsection{Optimality properties}
\label{sec:optimality}
Recall that a Solomonoff inductor on UTM $U$ will make no more error on $x$ than $K_U(x)$, i.e. the Kolmogorov complexity of $x$ on $U$.
Since HSI is equivalent to SolInd, it inherits this bound on individual sequences, but we can also generalise this bound to prediction of, and extrapolation from, entire datasets.

Consider an ordered dataset $X=(x_1, x_2,\dots x_{|X|})$ sampled from some semimeasure $\mu$ computed by $U_\mu$.
The probability assigned to $X$ by some predictor $P$ is
\[
    P(X) \coloneq \prod_{i=1}^{|X|} P(x_i \mid x_{<i})
\]
For HSI, this is equal to
\[
    P_\Eta(X) \coloneq \sum_{j\in\mathbb{N}} \Eta(U_j) \prod_{i=1}^{|X|} P_{U_j}(x_i \mid x_{<i})
\]
But since each individual SolInd prior is independent due to operating only on individual strings, $P_{U_j}(x_i \mid x_{<i})=P_{U_j}(x_i)$, giving
\begin{align*}
    P_\Eta(X) &= \sum_{j\in\mathbb{N}} \Eta(U_j) \prod_{i=1}^{|X|} P_{U_j}(x_i)\\
    P_\Eta(X) &= \sum_{j\in\mathbb{N}} \Eta(U_j) P_{U_j}(X)
\end{align*}

Similarly, we can generalise Equation \ref{eq:eta-sequence} to operate over datasets to obtain the conditioned hyperprior $\Eta(\cdot \mid X)$:
\begin{equation}
\label{eq:eta-dataset}
    \Eta(U_i \mid X) = \Eta(U_i) \prod_{x \in X} P_{U_i}(x)
\end{equation}
Note that this also gives $P_\Eta(X) = \sum_{i\in\mathbb{N}}\Eta(U_i \mid X)$.
Further note that every conditioned HSI prior is equivalent to $P_U$ for some UTM $U$ --- this means that when a hyperprior is conditioned, it is equivalent to a Solomonoff inductor updating its UTM so that it can better predict the given dataset.

Extrapolating from a dataset is therefore equivalent to induction using the conditioned hyperprior (analogous to inference using a trained model):
\begin{equation}
\label{eq:hsi-dataset}
    P_\Eta(x\mid X) \coloneq \sum_{j\in\mathbb{N}} \Eta(U_j \mid X) P_{U_j}(x)
\end{equation}

We can place a bound on HSI's excess error on $X$ relative to $U_\mu$, given by its complexity in the hyperprior.\footnote{Note that $U_\mu$ need not be universal and therefore may not be in $\Eta$'s domain; in this case we may choose any encoding $p$ of $U_\mu$ on any UTM $U$ and use $\Eta(U_\mu)=2^{-\ell(p)}\Eta(U)$.} We have:
\begin{align*}
P_\Eta(X)
    &= \sum_j \Eta(U_j) P_{U_j}(X)\\
    &> \Eta(U_\mu) P_{U_\mu}(X)\\
\frac{P_\Eta(X)}{P_{U_\mu}(X)}
    &> \Eta(U_\mu)
\end{align*}

Taking log loss on both sides, we obtain
\begin{equation}
\label{eq:hsi-error-bound}
-\log_2\left(\frac{P_\Eta(X)}{P_{U_\mu}(X)}\right) < -\log_2(\Eta(U_\mu))
\end{equation}
This bound is similar to SolInd's error bound on individual sequences given above, with both being guaranteed by the positive probability given to the generator of the sequence or dataset, and made strict by there existing infinitely many such generators on any UTM.

Since the above error bound applies regardless of $|X|$, HSI's average error on sequences in $X$ will approach 0 as it grows:
\begin{equation}
\label{eq:hsi-average-error-bound}
\frac{-\log_2\left( \frac{P_\Eta(X)}{P_{U_\mu}(X)} \right)}{|X|}
    < \frac{-\log_2(\Eta(U_\mu))}{|X|} \xrightarrow{|X|\to\infty} 0
\end{equation}

Further, as new sequences are sampled from $\mu$, we should expect $P_\Eta$'s excess error on any newly sampled $x_\nu$ to vanish:
\begin{equation}
\label{eq:hsi-convergence}
    E_{\mu} \left(-\log_2 \frac{P_{\Eta(\cdot\mid X)}(x_\nu)}
    {P_{U_\mu}(x_\nu)}\right)  \xrightarrow{|X|\to\infty} 0
\end{equation}

HSI's ability to converge to any computable distribution of sequences leads us to make the following claim:
\textbf{HSI is the ideal unbounded solution to machine learning over sets of sequences.}
While SolInd is the ideal prediction scheme over individual sequences, it cannot change its distribution nor extrapolate from a dataset in a manner guaranteed to be order-independent.
HSI is dependent on the choice of hyperprior in the same way that SolInd depends on the choice of UTM, but it can condition its hyperprior towards any distribution, remaining as powerful as SolInd at every point.

\section{Applications}

We propose two ways in which HSI may be of practical use:
\begin{enumerate}
\item \textbf{Conditioned hyperpriors as trained models.}
    As an unbounded model of sequence prediction by extrapolation from a dataset, HSI serves as an upper bound for what ML systems in this domain (including LLMs) can achieve, with conditioned hyperpriors describing ideal probability distributions for trained models.
    These systems may reasonably be expected to more closely resemble approximations of HSI as they improve. Past work has  investigated this hypothesis by comparing LLMs and other systems to SolInd~\citep{Grau-Moya2024, Young2025, Wan2024}, but as discussed above, SolInd can only operate over individual strings and so cannot model ML training.

\item \textbf{A formal model of document boundaries in training.}
    Most LLM training can be divided into two kinds: training on separate documents from a dataset, and training on concatenations of documents with a delimiter token indicating document boundaries.
    These two training methods are analogous to the conditioning of an HSI prior on a dataset and giving SolInd the same dataset as a concatenated prefix, respectively.
    \citet{Zhao2024} find that intra-document masking produces superior results to inter-document concatenation, in line with our claim that HSI's hyperprior conditioning is a superior approach to the use of concatenated datasets as prefixes.

\end{enumerate}

\subsection{Limitations}
Like SolInd, HSI is incomputable, requiring an unbounded number of programs to be executed for an unbounded amount of time and use an unbounded amount of memory.
However, since LLMs exist as remarkably capable sequence prediction algorithms with structural similarities to SolInd~\citep{Young2025}, we do not count this as an indication that either SolInd or HSI are not useful models.
Rather, our main concern with HSI (and equally with SolInd) is that it is \textit{monolithic}, considering hypotheses and machines separately and only by their outputs.
By contrast, bounded approximations of HSI might:
\begin{itemize}
    \item Recognise machines in $\Eta$ and programs on machines that make identical predictions and consolidate these
    \item Recognise programs that make similar predictions, and consolidate their similarities while separately modelling their differences
    \item Break computations into modular parts that may be reused and composed
    \item Compare its best-performing programs --- similarities may reveal learned knowledge about the structure of the data while differences may reflect uncertainties
    \item Learn to construct new programs or machines rather than including them all by default
\end{itemize}
Just as HSI does not improve SolInd's performance but instead creates new affordances like hyperprior conditioning and dataset extrapolation, new unbounded algorithms that model these or other optimisations may represent new ways of mapping the space of computable semimeasures and more closely resemble how state-of-the-art sequence prediction algorithms work in practice.

\section{Related Work}

\subsection{Unbounded Sequence Prediction}
\citet{Solomonoff1964} provided the first description of what he called `algorithmic probability', nowadays called Solomonoff Induction. He speculated that SolInd was ``about as good as [most prediction methods]''; \citet{Hutter2001} describes optimality results later found by Levin and Solomonoff and generalises these to arbitrary loss functions and alphabets.
A comprehensive treatment of algorithmic probability can be found in Chapter 4 of \citet{LiVitanyi2019}, while an in-depth philosophical perspective was given by \citet{Rathmanner2011}.

\subsection{UTM-Dependence}
Existing literature describes Solomonoff Induction as being highly dependent on the choice of UTM.
\citet{Wood2013} describe how some universally dominant semimeasures (equivalent to Solomonoff semimeasures on some TMs) are not equivalent to any universal mixtures of semimeasures (and therefore to any SolInd or HSI), while \citet{Leike2015} give examples of how any policy in any environment can be optimal according to SolInd on some UTM.
Their Lemma 1, which proves that the linear combination of an enumerable prior with a universal mixture results in another universal mixture, is closely related to our Theorem \ref{thm:enums-of-ums=ums}.
SolInd leading to arbitrarily-poor decisions does not contradict its optimality properties; SolInd's error bounds are finite but arbitrarily large on any particular sequence.
\citeauthor{Wood2013}'s result shows that Solomonoff semimeasures can be universal but not equivalent to SolInd, while Corollary \ref{cor:optimal-solind>hsi} shows that this is not true of HSI, which is (by Corollary \ref{cor:um=umum=solind=hsi}) an equivalent class to SolInd as long as its hyperprior includes all UTMs --- inclusion of any other subset of TMs does not change the class of priors considered over sequences or datasets.

This also gives a partial answer to the problem of the ``identification of `natural' Turing machines'' given by \citet{Hutter2009} --- if there exist subclasses of UTMs that have some properties desirable for induction, these classes must not assign arbitrary weights to UTMs or universal mixtures.
On the other hand, Corollary \ref{cor:utm-codes} shows that the space of UTMs is wider than defined by \citeauthor{Wood2013}.

\subsection{Meta-learning and Hierarchical Bayes}

\citet{Grau-Moya2024} train Transformers to perform meta-learning by training on data sampled from a UTM, with the aim of comparing them to the ideal of SolInd.
\citet{Baxter2000} frames meta-learning as induction over a prior on hypotheses --- that is, as hierarchical Bayesianism, with priors over observations and a hyperprior over these priors.
Under this lens, the ideal way for an ML system to meta-learn data from a UTM is to maintain a hyperprior over possible UTM distributions; this is an exact description of HSI.

\citet{Sterkenburg2018} examines the suitability of SolInd for prediction in the real world, finding that its incomputability is a fundamental obstacle to its practical use and that no computable approximation can have any of its optimality properties.
While the SolInd prior indeed cannot be learned in practice, HSI describes the ideal meta-learning algorithm for learning universal priors, and so may suggest ways in which approximations might be learned --- however, HSI's meta-learning is as incomputable as SolInd's learning, and so may only offer suggestions that are as impractical as SolInd's.

\subsection{Computable Probability Measures}

\citet{deFinetti1931} describes how any infinite series of exchangeable observations must be conditionally independent in relation to some latent variable, and that the distribution defined by this variable must be expressible as a mixture over possible distributions --- in other words, a hierarchical Bayesian structure inevitably emerges when modelling any exchangeable data.
\citet{HewittSavage1955} describe how de Finetti's theorem applies to Cartesian products such as the Cantor space over which sequence prediction is performed, while \citet{FreerRoy2009, FreerRoy2012} prove that de Finetti mixtures over computable sequences of real numbers are themselves computable.
The generalisation of computable de Finetti measures to sequence prediction is an open research area with a close theoretical relation to this work --- we expect that de Finetti mixtures of enumerable semimeasures over sequences can be proven to also be enumerable semimeasures. If these mixtures are universal, they are (by Theorem \ref{thm:enums-of-ums=ums}) exactly equivalent to HSI, which would make HSI the unique ideal solution to extrapolation from exchangeable sets of sequences.

\subsection{LLMs as Approximations of SolInd}
Recent work has examined the hypothesis that LLMs can be understood as approximations of SolInd, and that they form better approximations than other architectures.
\citet{Deletang2023} argue that language modelling is equivalent to compression, i.e. finding the shortest description of the training dataset according to some encoding. SolInd and HSI over a dataset use weighted averages of all possible descriptions, but shorter inputs get exponentially more weight and therefore dominate.
\citet{Grau-Moya2024} trained Transformer models alongside other neural networks on data sampled from a UTM, finding that Transformers outperform all other models on this data, providing useful evidence that their resemblance to SolInd extends beyond language modelling to sequence modelling in general.
Transformers' performance suffered on out-of-distribution sequence lengths, but this can be attributed to the positional encoding used~\citep{genewein2026algorithmic}.
\citet{Wan2024} examine LLMs as approximations of SolInd, finding that when Transformers are fine-tuned on training examples on which they performed poorly, they improve more than when fine-tuned on examples on which they performed well, since these examples provided the largest Bayesian update to the model's priors.
HSI allows us to quantify this update exactly: updates to each $\Eta(U_i)$ according to an example $x$ are directly proportional to $P_{U_i}(x)$. Since $0<P_{U_i}(x)<1$, this is always a decrease, so each $\Eta(U_i)$ receives a greater update when $P_{U_i}(x)$ is lesser, and likewise with $\Eta$ receiving larger updates overall with smaller $P_\Eta(x)$.

\section{Conclusion}
We have presented \textit{Hierarchical Solomonoff Induction}, or HSI, an ideal unbounded sequence prediction algorithm with bounded error and convergence guarantees over both datasets and sequences.
HSI's hyperprior $\Eta$ may be defined over any enumeration of Turing machines as long as all universal Turing machines are included, with the class of HSI priors $\mathcal{U}_\Eta$ being identical to the class of SolInd priors $\mathcal{U}_M$ for any such enumeration.
HSI is identical to SolInd over individual sequences, and is additionally capable of describing ideal sequence prediction by extrapolation from a dataset in a way that SolInd is not.
We hope that HSI will be useful for understanding how LLMs and other sequence prediction models work by providing an upper bound on what they may achieve and an idealised method that they may approach.

\section*{Acknowledgments}
An AI assistant helped plan, review, and edit this paper, as well as contributing an initial proof that a Turing machine could emulate any mixture of universal Turing machines. The concepts behind this proof were generalised to produce Theorem \ref{thm:enums-of-ums=ums}.
No prose or results in this paper are AI-generated.

I would like to thank my PhD supervisors, Michael Witbrock and Robert Amor.
This research was supported by the University of Auckland Doctoral Scholarship.

\bibliography{references}

\end{document}